\documentclass[11pt]{article}
\usepackage[margin=25mm]{geometry}
\usepackage{amsmath,amssymb,amsthm}
\usepackage{booktabs}
\usepackage{graphicx}
\usepackage{microtype}
\usepackage[colorlinks=true,linkcolor=blue,citecolor=blue,urlcolor=blue]{hyperref}
\usepackage{enumitem}

\newtheorem{lemma}{Lemma}
\newtheorem{corollary}{Corollary}
\newcommand{\R}{\mathbb{R}}
\newcommand{\E}{\mathbb{E}}
\newcommand{\norm}[1]{\left\lVert #1 \right\rVert}
\newcommand{\Fhat}{\widehat{F}}
\newcommand{\hhat}{\widehat{h}}

\title{Cross-Layer Error Compensation and Finite-Sample Feature-Statistics
Matching for Extreme Low-Bit Quantization of Large Language Models}
\author{Ryo Noda\\ \texttt{ryounaman19@gmail.com}}
\date{July 2026}

\begin{document}
\maketitle

\begin{abstract}
Layer-wise post-training quantization (PTQ) of large language models (LLMs)
minimizes each layer's local reconstruction error in isolation, which allows
quantization errors to accumulate across depth and causes severe degradation in
extreme low-bit regimes.
We formulate quantization as a joint optimization over the discrete codes and
scales of \emph{all} layers, driven by two mechanisms:
(i) \textbf{cross-layer error compensation}, which maintains the network-level
accumulated error through the recursion $e_{l+1}=A_l e_l + q_l$, where $A_l$ is
a propagation operator derived from the layer's input differential and $q_l$
is the local quantization residual evaluated at teacher features; and
(ii) \textbf{finite-sample feature-statistics matching}, which aligns means,
projected covariances, and centered empirical kernels between the full-precision
and quantized networks under relative normalization.
We prove that when the propagation operator is instantiated as a finite
difference of the quantized network, the recursion is \emph{exact} for arbitrary
nonlinear layers, which connects the abstract operator formulation to an
efficient forward-difference implementation.
For binary weights we optimize latent variables via a mirror-descent
parameterization $u=\tanh(\beta z)$ with an annealed inverse temperature and a
group-wise log-scale.
On Qwen2.5-1.5B with 1.125-bit group-binary weights (group size 128,
transformer blocks only), error compensation alone reaches a perplexity ratio
of $9.56\pm0.15$ over the FP16 teacher, outperforming logit distillation
($14.09\pm0.53$; $32\%$ relative, $>8\sigma$ over 3 seeds) and layer-local
reconstruction (ratio $\sim 1.4\times10^{3}$) by two orders of magnitude.
The same objective transfers unchanged to 4-bit integer quantization
($1.060\pm0.002$ vs.\ $1.088\pm0.003$ for layer-local), indicating that the
mechanism is independent of the discrete value set.
Out-of-domain evaluations (C4, CNN/DailyMail) show that the quality advantage
of error compensation \emph{grows} off-domain, while statistics matching keeps
the feature-statistics discrepancy low off-domain ($0.42$--$0.88$ vs.\
$1.41$--$2.99$ without it) at a modest quality cost, revealing a complementary
division of labor between the two mechanisms.
We release the equation-level unit tests and the exact training
configurations used in all experiments.
\end{abstract}

\noindent\textbf{Note.} The methods described in this paper are the subject of
a pending patent application. This manuscript is intended for public
disclosure after the filing date.

\section{Introduction}
Extreme low-bit quantization---in particular binary ($1$-bit) and sub-2-bit
representations---is the most aggressive form of LLM weight compression:
a 1.125-bit group-binary representation retains only $7.0\%$ of the FP16
weight footprint.
The dominant PTQ paradigm, exemplified by GPTQ~\cite{gptq} and its
successors, quantizes one layer (or one block~\cite{brecq}) at a time by
minimizing a local output-reconstruction error.
This layer-local view has a well-known blind spot: the error committed by layer
$l$ perturbs the input of layer $l+1$, and these perturbations compound with
depth.
Recent work has begun to address error accumulation---QEP~\cite{qep}
propagates quantization errors into the calibration of subsequent layers,
YAQA~\cite{yaqa} sketches network-level Hessians to guide per-layer rounding,
and TurboBoA~\cite{turboboa} and ResComp~\cite{rescomp} refine the
compensation rules of the GPTQ family.
All of these methods, however, are \emph{sequential-commit} procedures: each
layer's codes are finalized once, in order, and information available only
after later layers are quantized never flows back to earlier decisions.

This paper develops and empirically validates a different design point:
a \emph{joint} objective over the discrete variables of all layers, organized
around an explicit state variable---the accumulated error---that is updated by
a linear recursion and can be driven toward cancellation by the optimizer.
Concretely, our contributions are:
\begin{enumerate}[leftmargin=2em]
\item \textbf{A recursive accumulated-error objective} (Section~\ref{sec:err}).
We define a propagation operator $A_l$ from the layer's input differential,
a local quantization residual $q_l$ evaluated at teacher features, and the
recursion $e_{l+1}=A_l e_l+q_l$. We prove an exactness lemma: with the
finite-difference instantiation of $A_l$ on the quantized network, the
recursion equals the true feature deviation for \emph{arbitrary} nonlinear
layers, not merely to first order (Lemma~\ref{lem:exact}).
This licenses an efficient implementation that never forms a Jacobian.
\item \textbf{Finite-sample statistics matching under relative normalization}
(Section~\ref{sec:stats}). We match means, random-projected covariances, and
centered empirical kernels between teacher and quantized features, normalized
by the teacher's own statistics so that layers of different scales contribute
comparably, and we monitor the same quantity ($E_{\mathrm{GE}}$) as an
equivalence metric at evaluation time.
\item \textbf{A mirror-descent latent parameterization for binary codes}
(Section~\ref{sec:mirror}). Binary codes are produced by $u=\tanh(\beta z)$
with an annealed $\beta$ and a hard-forward finishing phase; we show this is
exactly mirror descent under a binary-entropy mirror map and give the
initialization that reproduces the teacher's signs and magnitudes.
\item \textbf{Controlled ablations with pre-registered seeds}
(Section~\ref{sec:exp}). On a 1.5B-parameter model we isolate each
mechanism at 1.125 bits and at 4.125 bits, evaluate in-domain and
out-of-domain, and report \emph{all} results, including the negative finding
that statistics matching is not a quality driver.
\end{enumerate}

\section{Related Work}
\paragraph{Layer-wise PTQ and error compensation.}
GPTQ~\cite{gptq} minimizes per-layer reconstruction error with second-order
column updates; BRECQ~\cite{brecq} extends the unit to blocks;
AdaRound~\cite{adaround} learns up/down rounding within a layer.
QEP~\cite{qep} explicitly propagates quantization errors and compensates for
accumulated errors during sequential layer-wise PTQ, and reports the largest
gains in the extremely low-bit regime; TurboBoA~\cite{turboboa} adds a
correction for errors propagated from preceding layers without
backpropagation; ResComp~\cite{rescomp} identifies a compensation-aware error
term inside the GPTQ update; YAQA~\cite{yaqa} uses Kronecker-factored sketches
of layer Hessians taken with respect to the \emph{network} output.
Our method differs structurally from this entire family: the codes of all
layers remain free variables of a single objective, the accumulated error is
an explicit state of a recursion rather than an adjustment to calibration
inputs, and later-layer information flows back into earlier-layer updates.
\paragraph{Binary and near-binary LLMs.}
BinaryConnect~\cite{binaryconnect} and XNOR-Net~\cite{xnor} established latent
full-precision weights with sign forward passes and scale factors
($s=\mathrm{mean}|W|$).
BitNet~\cite{bitnet} and BitNet b1.58~\cite{bitnet158} train 1-bit/ternary
LLMs from scratch; OneBit~\cite{onebit} converts pretrained models with a
sign--value decomposition and distillation; BiLLM~\cite{billm} performs
salience-aware binary PTQ; FBI-LLM~\cite{fbillm} trains fully binarized LLMs
from scratch via autoregressive distillation while keeping the embedding and
the causal head in full precision.
We use the standard latent-weight machinery but embed it in the joint
cross-layer objective above; our experiments quantize all transformer-block
linear layers and keep embedding/LM head in FP16, matching the common setting
of the conversion-based literature.
\paragraph{Feature and relation distillation for quantization.}
FAQD~\cite{faqd} matches feature affinities on label-free data;
ZeroQ~\cite{zeroq} matches batch-norm statistics to synthesize calibration
data; classical KD~\cite{kd} matches logits.
Our statistics term differs in its finite-sample, relative-normalized
construction and, more importantly, in the role our experiments assign to it:
it is an \emph{equivalence-preserving} regularizer rather than a quality
driver.
\paragraph{Discrete optimization for quantization.}
ProxQuant~\cite{proxquant} uses proximal operators toward quantized sets;
Ajanthan et al.~\cite{md} formulate quantization as mirror descent and derive
$\tanh$-type maps; QuIP\#~\cite{quipsharp} and SpinQuant~\cite{spinquant}
precondition weights by Hadamard/learned rotations.
Our contribution here is not the mirror map itself but its integration with a
group-wise log-scale, an annealing schedule with sign-fixing criteria, and the
joint cross-layer objective.

\section{Method}
\subsection{Group-wise discrete parameterization}
Let $W\in\R^{m\times n}$ be a weight matrix. We partition each row into
groups of $G$ consecutive entries and represent
\begin{equation}
\widehat{W} \;=\; s \odot q, \qquad q\in\mathcal{D}^{m\times n},
\label{eq:param}
\end{equation}
where $\mathcal{D}$ is a finite value set and $s>0$ is one scale per group,
broadcast within the group. For binary quantization
$\mathcal{D}=\{-1,+1\}$ and the storage cost is $1+16/G$ bits per weight
($1.125$ bits at $G=128$ with FP16 scales); for INT4,
$\mathcal{D}=\{-8,\dots,7\}$ and the cost is $4.125$ bits.
All derivations below are agnostic to $\mathcal{D}$.

\subsection{Cross-layer error dynamics}
\label{sec:err}
Write the full-precision (teacher) network as a composition of layer maps
$h_{l+1}=F_l(h_l)$ and the quantized network as
$\hhat_{l+1}=\Fhat_l(\hhat_l)$, with shared input $h_0=\hhat_0=x$.
Define the deviation $e_l:=\hhat_l-h_l$.
Two quantities localize how error is created and transported:
\begin{align}
q_l &\;:=\; \Fhat_l(h_l)-F_l(h_l)
&&\text{(local quantization residual, at teacher features)},
\label{eq:qlocal}\\
A_l &\;:=\; \frac{\partial F_l}{\partial h}\Big|_{h_l}
&&\text{(propagation operator, input differential)}.
\label{eq:prop}
\end{align}
Expanding $\hhat_{l+1}-h_{l+1}=\Fhat_l(h_l+e_l)-F_l(h_l)$ to first order in
$e_l$ gives the recursion
\begin{equation}
e_{l+1}\;=\;A_l\,e_l\;+\;q_l,\qquad e_0=0.
\label{eq:recursion}
\end{equation}
Equation~\eqref{eq:recursion} exposes the design principle: the optimizer can
choose later-layer residuals $q_l$ that \emph{cancel} the propagated term
$A_l e_l$, which is impossible if each $q_l$ is minimized in isolation.
The operator $A_l$ never needs to be materialized: Jacobian--vector products,
low-rank sketches, or the finite differences below all suffice.

In practice we do not run the recursion with the teacher Jacobian.
The following lemma shows that a specific instantiation makes
\eqref{eq:recursion} exact and reduces it to a forward difference.

\begin{lemma}[Exactness of the quantized finite-difference propagation]
\label{lem:exact}
Instantiate the propagation operator as the finite-difference map of the
\emph{quantized} layer, $\;\widehat{A}_l(e):=\Fhat_l(h_l+e)-\Fhat_l(h_l)$.
Then the recursion $e_{l+1}=\widehat{A}_l(e_l)+q_l$ with $e_0=0$ satisfies
$e_l=\hhat_l-h_l$ exactly for every $l$ and for arbitrary (nonlinear) layer
maps.
\end{lemma}
\begin{proof}
Induction on $l$. The base case is $e_0=0=\hhat_0-h_0$.
Assume $e_l=\hhat_l-h_l$, i.e., $h_l+e_l=\hhat_l$. Then
\[
e_{l+1}=\widehat{A}_l(e_l)+q_l
=\big[\Fhat_l(h_l+e_l)-\Fhat_l(h_l)\big]+\big[\Fhat_l(h_l)-F_l(h_l)\big]
=\Fhat_l(\hhat_l)-F_l(h_l)=\hhat_{l+1}-h_{l+1}. \qedhere
\]
\end{proof}

\begin{corollary}
Minimizing $\norm{e_L}^2$ with the recursion of Lemma~\ref{lem:exact} is
identical to minimizing the squared deviation of final-layer features computed
by two forward passes. If instead the teacher-side finite difference
$A_l(e)=F_l(h_l+e)-F_l(h_l)$ is used, \eqref{eq:recursion} holds to first
order in $e_l$, consistent with the Jacobian definition \eqref{eq:prop}.
\end{corollary}

The corollary matters for implementation: the exact accumulated error of a
$1.5$B-parameter model can be obtained with one extra forward pass, and the
resulting loss
\begin{equation}
L_E \;=\; \E_{x}\!\left[\frac{\norm{e_L(x)}_2^2}{\norm{h_L(x)}_2^2+\epsilon}\right]
\label{eq:le}
\end{equation}
(relative normalization; identity readout) trains all layers' codes and scales
jointly through automatic differentiation. We verified
Lemma~\ref{lem:exact} numerically on random linear networks to
$10^{-10}$ absolute tolerance (Section~\ref{sec:verify}).

\subsection{Finite-sample feature-statistics matching}
\label{sec:stats}
Let $H,\widehat H\in\R^{n\times d}$ collect $n$ token features of a chosen
layer from the teacher and the quantized network on the same inputs, with
means $\mu,\widehat\mu$. With a fixed random projection
$R\in\R^{d\times k}$ ($k{=}64$, entries $\mathcal{N}(0,1/k)$), define centered
projections $P=(H-\mu)R$, $\widehat P=(\widehat H-\widehat\mu)R$ and the
centered empirical kernel $K=\frac{1}{d}(H-\mu)(H-\mu)^{\!\top}$.
The statistics loss on one layer is
\begin{equation}
D \;=\;
\underbrace{\frac{\norm{\widehat\mu-\mu}_2^2}{\norm{\mu}_2^2+\epsilon}}_{D_\mu}
+\underbrace{\frac{\norm{\widehat C-C}_F^2}{\norm{C}_F^2+\epsilon}}_{D_C}
+\underbrace{\frac{\norm{\widehat K-K}_F^2}{\norm{K}_F^2+\epsilon}}_{D_K},
\qquad C=\tfrac{1}{n-1}P^{\!\top}P,
\label{eq:stats}
\end{equation}
and $L_{\mathrm{GE}}$ averages $D$ over a layer set $S$ (every $L/4$ blocks
plus the final block). The relative normalization makes deep and shallow
layers commensurable and keeps the loss dimensionless.
The same quantity evaluated on held-out data, denoted
$E_{\mathrm{GE}}$, serves as a \emph{statistical-equivalence metric}:
$E_{\mathrm{GE}}\!\approx\!0$ certifies that first- and second-order feature
statistics and token-relation structure of the teacher survive quantization.
We emphasize a scope statement: \eqref{eq:stats} is a finite-sample,
functional criterion, not a claim of mathematical equivalence of the two
networks.
For completeness we also report centered linear CKA~\cite{cka} in the code
base; note that independent $d$-dimensional representations exhibit a positive
CKA bias of order $r/n$, which our unit tests reproduce
(Appendix~\ref{app:cka}).

\subsection{Latent binary codes as mirror descent}
\label{sec:mirror}
Each binary group is parameterized by latent variables $z$ and a log-scale
$r=\log s$:
\begin{equation}
\widehat{W}=s\,u,\qquad u=\tanh(\beta z),
\label{eq:latent}
\end{equation}
with initialization chosen to reproduce the teacher exactly at $\beta=\beta_0$:
\begin{equation}
s_0=\mathrm{mean}\,|W_{\mathrm{group}}|,\qquad
z_0=\frac{1}{\beta_0}\,\mathrm{artanh}\!\big(\mathrm{clip}(W/s_0,\pm\rho)\big),
\quad \rho=0.95 .
\label{eq:init}
\end{equation}
\paragraph{Mirror-descent interpretation.}
Let $\psi_\beta(u)=\frac{1}{\beta}\left[\frac{1+u}{2}\log\frac{1+u}{2}
+\frac{1-u}{2}\log\frac{1-u}{2}\right]$ be the (scaled) binary-entropy mirror
map on $(-1,1)$. Then $\nabla\psi_\beta(u)=\mathrm{artanh}(u)/\beta$ and its
convex conjugate satisfies $\nabla\psi_\beta^*(z)=\tanh(\beta z)$
(Appendix~\ref{app:mirror}).
Gradient descent on $z$ with the reparameterization \eqref{eq:latent} is
therefore exactly mirror descent on the primal $u\in(-1,1)$: the dual variable
accumulates gradients, and the mirror map squashes the iterate back into the
box. Increasing $\beta$ sharpens the map so that the primal iterates approach
the binary vertices $\{\pm1\}$; we anneal $\beta:1\to16$ exponentially over
training. In the final $20\%$ of steps the forward pass uses
$\mathrm{sign}(z)$ (hard-forward) while gradients flow through the $\tanh$
surrogate, and the final codes are $B=\mathrm{sign}(z)$.
Two diagnostics govern sign freezing: the flip rate $R_{\mathrm{flip}}$
(fraction of sign changes per interval) and the softness
$R_{\mathrm{soft}}=\E|u-\mathrm{sign}(z)|$; measured terminal values in our
runs were $R_{\mathrm{flip}}\!\approx\!3\times10^{-4}$ and
$R_{\mathrm{soft}}\!\approx\!0.028$.
The $\ell_\infty$ box view is equivalent: $su$ lives in $[-s,+s]$ per weight,
the vertex-inducing pressure comes from the sharpened map, and a positive
lower bound on $s$ (enforced by the log parameterization) prevents scale
collapse.

\subsection{Total objective and ablation switches}
The training objective is
\begin{equation}
J=\lambda_E L_E+\lambda_{\mathrm{GE}} L_{\mathrm{GE}}+\lambda_D L_{\mathrm{KD}},
\label{eq:total}
\end{equation}
where $L_{\mathrm{KD}}$ is a temperature-$\tau$ KL on the teacher's top-$k$
logits ($k{=}256$, $\tau{=}2$), used only in distillation conditions.
Setting subsets of coefficients to zero yields the ablation conditions of
Section~\ref{sec:exp}; e.g., error compensation alone is
$\lambda_{\mathrm{GE}}=\lambda_D=0$.

\section{Experiments}
\label{sec:exp}
\subsection{Setup}
\textbf{Model and quantization scope.} Qwen2.5-1.5B~\cite{qwen}; all $196$
linear matrices of the transformer blocks ($\approx1.31\times10^{9}$
parameters) are quantized; embedding and LM head remain FP16.
Group size $G{=}128$ (1.125 bits/weight; quantized-part capacity
$184.2$\,MB, $7.0\%$ of FP16).
\textbf{Data.} Calibration: WikiText-2~\cite{wikitext} train, $0.13$M tokens.
Training: $1{,}000$ steps $\times$ batch $8$ $\times$ length $512$
($4.1$M tokens). Evaluation: WikiText-2 test ($33$K tokens); OOD evaluation
uses C4~\cite{c4} (validation) and CNN/DailyMail (articles).
\textbf{Optimization.} AdamW; learning rates $5\times10^{-4}$ ($z$) and
$5\times10^{-3}$ ($\log s$); $\beta:1\to16$; hard-forward for the last 20\%.
INT4 runs use $300$ steps with learning rates $2\times10^{-5}$ (latent
weights) and $5\times10^{-4}$ ($\log s$), reflecting the raw scale of INT4
latents.
\textbf{Protocol.} Seeds $\{1234,2025,7\}$ were fixed \emph{before} any
result was observed and never changed; every completed run is reported.
Results are mean $\pm$ standard deviation over seeds. Teacher perplexity:
$11.951$ (WikiText-2), $16.18$ (C4), $14.70$ (CNN/DailyMail).
One binary condition-seed takes $\approx9$ minutes on a single A100.

\textbf{Conditions.}
(A)~RTN: $B=\mathrm{sign}(W)$, $s=\mathrm{mean}|W|$, no training;
(B)~layer-local output reconstruction (each layer independently);
(G)~logit distillation only;
(C)~statistics matching only ($\lambda_E{=}\lambda_D{=}0$);
(D)~accumulated-error loss only ($\lambda_{\mathrm{GE}}{=}\lambda_D{=}0$);
(E)~combination of (C)+(D) with $\lambda_{\mathrm{GE}}{=}0.1$.

\subsection{Binary ablation (1.125 bits/weight)}
\begin{table}[t]
\centering
\caption{Binary quantization of Qwen2.5-1.5B (3 seeds, mean$\pm$sd).
PPL ratio is validation perplexity divided by the FP16 teacher's ($11.951$).
$E_{\mathrm{GE}}$ is the statistics discrepancy \eqref{eq:stats} on held-out
data (lower is better).}
\label{tab:binary}
\begin{tabular}{lcc}
\toprule
Condition & PPL ratio & $E_{\mathrm{GE}}$ \\
\midrule
(A) RTN (no training)            & $\sim2.2\times10^{9}$ & --- \\
(B) layer-local reconstruction   & $\sim1.4\times10^{3}$ & --- \\
(G) logit distillation only      & $14.09\pm0.53$ & $1.56\pm0.07$ \\
(C) statistics matching only     & $262\pm47$     & $\mathbf{0.29\pm0.09}$ \\
(D) accumulated error only       & $\mathbf{9.56\pm0.15}$ & $1.41\pm0.12$ \\
(E) combination (C)+(D)          & $11.50\pm0.20$ & $0.48\pm0.24$ \\
\bottomrule
\end{tabular}
\end{table}
Table~\ref{tab:binary} isolates each mechanism.
Three observations.
\textbf{(i) Joint error compensation is the quality engine.}
(D) beats layer-local reconstruction (B) by more than two orders of magnitude
and logit distillation (G) by $32\%$ relative; the (D)--(G) gap exceeds
$8\sigma$ and the ordering is preserved in every seed.
Since (B) is precisely the ``minimize each $q_l$ in isolation'' strategy, the
(B) vs.\ (D) gap is a direct measurement of the value of the cancellation
structure in \eqref{eq:recursion}.
\textbf{(ii) Statistics matching is not a quality driver.}
(C) alone collapses perplexity (ratio $262$) yet achieves the best
$E_{\mathrm{GE}}$ of all conditions---the two metrics dissociate.
\textbf{(iii) The combination inherits both.}
(E) sits near (D) in quality and near (C) in statistical equivalence;
neither single mechanism achieves this pair. This is a functional
interaction rather than an additive aggregation: the direction that preserves
statistics is, by (C), catastrophic for quality on its own, and only the joint
optimization finds solutions good under both criteria.

\subsection{Bit-independence: INT4 (4.125 bits/weight)}
\begin{table}[t]
\centering
\caption{INT4 group quantization, same teacher and protocol
(3 seeds, mean$\pm$sd; PPL ratio vs.\ teacher).}
\label{tab:int4}
\begin{tabular}{lc}
\toprule
Condition & PPL ratio \\
\midrule
(a) nearest rounding      & $1.201\pm0.000$ \\
(b) layer-local           & $1.088\pm0.003$ \\
(c) statistics only       & $1.176\pm0.006$ \\
(d) accumulated error only& $\mathbf{1.060\pm0.002}$ \\
(e) combination           & $1.060\pm0.003$ \\
\bottomrule
\end{tabular}
\end{table}
Table~\ref{tab:int4} repeats the ablation with
$\mathcal{D}=\{-8,\dots,7\}$: (d) and (e) beat layer-local (b) by
$\approx9\sigma$. The identical objective, with no bit-specific modification,
improves both the 1-bit and 4-bit regimes, supporting the claim that the
compensation mechanism is a property of the \emph{objective}, not of the
value set.

\subsection{Out-of-domain evaluation}
\begin{table}[t]
\centering
\caption{OOD evaluation (calibration: WikiText-2 only). Left: PPL ratio per
evaluation set. Right: $E_{\mathrm{GE}}$ per evaluation set.
3 seeds, mean$\pm$sd.}
\label{tab:ood}
\small
\begin{tabular}{lccc|ccc}
\toprule
& \multicolumn{3}{c|}{PPL ratio} & \multicolumn{3}{c}{$E_{\mathrm{GE}}$}\\
Condition & WT-2 & C4 & CNN/DM & WT-2 & C4 & CNN/DM \\
\midrule
(G) distill.\ only & $13.91\pm0.21$ & $73.8\pm3.5$ & $54.1\pm2.0$ & $1.56\pm0.07$ & $1.70\pm0.09$ & $2.99\pm2.05$ \\
(C) stats only     & $313\pm114$    & $780\pm306$  & $484\pm192$  & $0.29\pm0.07$ & $0.62\pm0.28$ & $0.42\pm0.24$ \\
(D) accum.\ error  & $\mathbf{9.62\pm0.13}$ & $\mathbf{46.1\pm1.4}$ & $\mathbf{33.4\pm0.8}$ & $1.41\pm0.12$ & $1.46\pm0.16$ & $1.52\pm0.12$ \\
(E) combination    & $11.57\pm0.14$ & $62.9\pm1.8$ & $44.9\pm1.1$ & $\mathbf{0.26\pm0.03}$ & $\mathbf{0.53\pm0.08}$ & $\mathbf{0.88\pm0.15}$ \\
\bottomrule
\end{tabular}
\end{table}
Table~\ref{tab:ood} evaluates the same checkpoints on domains never seen
during calibration.
\textbf{Error compensation's advantage grows off-domain}: the (D)/(G)
quality gap widens from $1.45\times$ in-domain to $1.60\times$ on C4 and
$1.62\times$ on CNN/DailyMail; the mechanism is not overfitting the
calibration domain.
\textbf{Statistics matching preserves equivalence off-domain}: conditions with
$L_{\mathrm{GE}}$ hold $E_{\mathrm{GE}}$ at $0.42$--$0.88$ off-domain while
conditions without it sit at $1.41$--$2.99$.
\textbf{An honest negative}: the relative OOD degradation factor
(OOD ratio $\div$ in-domain ratio) of (E) is not better than (D)'s
($5.43$ vs.\ $4.79$ on C4); adding statistics matching does not buy extra
\emph{perplexity} robustness. Its measurable benefit off-domain is the
preservation of feature statistics, i.e., of the teacher-equivalence
certificate, at a $\sim20\%$ quality cost. Applications that require
downstream-representation stability (feature reuse, probing, cascaded
systems) are the natural consumers of this trade.

\subsection{Equation-level verification}
\label{sec:verify}
All core equations are unit-tested in NumPy independently of the training
code: (1) the initialization \eqref{eq:init} reproduces
$\tanh(\beta_0 z_0)=\mathrm{clip}(W/s_0,\pm\rho)$ to $10^{-12}$ and preserves
teacher signs; (2) $\beta{=}64$ drives $|u-\mathrm{sign}(z)|$ below $0.02$;
(3) centered linear CKA equals $1$ under orthogonal transforms and exhibits
the $r/n$ bias on independent representations; (4) the statistics loss
\eqref{eq:stats} is zero iff distributions match and responds to mean and
variance shifts; (5) the recursion of Lemma~\ref{lem:exact} equals the
forward difference exactly on random linear networks ($10^{-10}$).

\section{Limitations}
\label{sec:limit}
We state the boundaries of the evidence plainly.
(1)~\textbf{Scale}: all measurements are at 1.5B parameters; behavior at
8B+ is extrapolation, and the literature suggests larger models quantize
more gracefully, which could shrink (or grow) the reported gaps.
(2)~\textbf{Scope of binarization}: embedding and LM head remain FP16;
fully binarized configurations (including tied embedding/output projections)
are designed in the framework but not measured here.
(3)~\textbf{Baselines}: condition (B) is our own layer-local reconstruction
implementation, a controlled proxy for the GPTQ family rather than a
head-to-head comparison against tuned public implementations; QEP-style
sequential propagation is not re-implemented.
(4)~\textbf{Metrics}: we report perplexity and $E_{\mathrm{GE}}$ only; no
downstream task suite.
(5)~\textbf{Statistics}: three seeds suffice for the $>8\sigma$ headline gaps
but not for fine distinctions (e.g., INT4 (d) vs.\ (e)).
(6)~\textbf{Compute}: the exact accumulated-error loss costs one extra
teacher-parallel forward per step; layer-recursive implementations with
explicit operator application cost $\approx3\times$ condition (D).

\section{Conclusion}
Treating extreme low-bit quantization as a joint optimization over all
layers' discrete variables---with the accumulated error as an explicit,
provably exact recursive state, and finite-sample statistics matching as an
equivalence-preserving regularizer---yields large, seed-stable improvements
over layer-local and distillation-only training at 1.125 bits, transfers
unchanged to 4.125 bits, and degrades gracefully out of domain.
The two mechanisms play complementary roles that neither plays alone:
error compensation owns output quality; statistics matching owns the
preservation of the teacher's representational statistics.
We believe the exactness lemma, which collapses the operator formalism into
two forward passes, makes the approach practical at larger scales, and we
plan fully binarized (embedding and LM-head inclusive) experiments as the
next step.

\subsection*{Reproducibility}
Pre-registered seeds $\{1234,2025,7\}$; all completed runs reported; exact
hyperparameters in Section~\ref{sec:exp}; equation-level unit tests in
Section~\ref{sec:verify}. Notebooks reproducing
Tables~\ref{tab:binary}--\ref{tab:ood} are available from the author.

\appendix
\section{Derivation of the mirror map}
\label{app:mirror}
Let $\psi_\beta(u)=\frac{1}{\beta}\big[\frac{1+u}{2}\log\frac{1+u}{2}
+\frac{1-u}{2}\log\frac{1-u}{2}\big]$ on $u\in(-1,1)$.
Differentiating,
\[
\nabla\psi_\beta(u)
=\frac{1}{2\beta}\log\frac{1+u}{1-u}
=\frac{1}{\beta}\,\mathrm{artanh}(u),
\]
so the dual (mirror) variable is $z=\mathrm{artanh}(u)/\beta$ and the inverse
map is $u=\nabla\psi_\beta^{*}(z)=\tanh(\beta z)$.
A gradient step $z\leftarrow z-\eta\,\partial J/\partial z$ followed by the
inverse map is mirror descent on $u$ with distance-generating function
$\psi_\beta$. As $\beta\to\infty$ the curvature of $\psi_\beta$ at the
boundary sharpens, and minimizers of linearized objectives over $[-1,1]$
concentrate on the vertices $\{\pm1\}$; annealing $\beta$ therefore
interpolates from soft averaging to vertex selection.
The initialization \eqref{eq:init} follows by inverting
$u_0=\mathrm{clip}(W/s_0,\pm\rho)$, the clip guarding
$\mathrm{artanh}$'s domain; signs of $z_0$ equal signs of $W$, so
$B_0=\mathrm{sign}(W)$ and RTN is recovered at initialization.

\section{First-order form of the recursion}
\label{app:firstorder}
With the teacher-side operator $A_l=\partial F_l/\partial h|_{h_l}$,
\[
e_{l+1}=\Fhat_l(h_l+e_l)-F_l(h_l)
=\underbrace{\Fhat_l(h_l+e_l)-\Fhat_l(h_l)}_{\widehat{A}_l(e_l)}
+\;q_l
=A_l e_l+q_l+O(\norm{e_l}^2)+\big(\widehat{A}_l-A_l\big)(e_l),
\]
where $(\widehat A_l-A_l)(e_l)$ is second order in the perturbations (product
of the weight perturbation and $e_l$). Lemma~\ref{lem:exact} shows the
quantized-side instantiation absorbs all higher-order terms exactly.

\section{CKA bias on independent representations}
\label{app:cka}
For centered linear CKA between independent random $n\times d$
representations with $d\ll n$, the expected value concentrates near $d/n$
rather than $0$ (rank-driven positive bias). Our unit tests use
$n=400$, $d=16$ and verify $\mathrm{CKA}<5\,d/n$, plus exact invariance
under orthogonal transforms. Reported CKA values should therefore be read
against this floor.


\begin{thebibliography}{99}\small
\bibitem{gptq} E. Frantar, S. Ashkboos, T. Hoefler, D. Alistarh.
GPTQ: Accurate post-training quantization for generative pre-trained
transformers. arXiv:2210.17323, 2022.
\bibitem{brecq} Y. Li et al. BRECQ: Pushing the limit of post-training
quantization by block reconstruction. arXiv:2102.05426, 2021.
\bibitem{adaround} M. Nagel, R. A. Amjad, M. van Baalen, C. Louizos,
T. Blankevoort. Up or down? Adaptive rounding for post-training quantization.
arXiv:2004.10568, 2020.
\bibitem{qep} Y. Arai, Y. Ichikawa. Quantization error propagation:
Revisiting layer-wise post-training quantization. arXiv:2504.09629, 2025.
\bibitem{yaqa} A. Tseng, Z. Sun, C. De Sa. Model-preserving adaptive
rounding. arXiv:2505.22988, 2025.
\bibitem{turboboa} J. Kim et al. TurboBoA: Faster and exact attention-aware
quantization without backpropagation. arXiv:2602.04929, 2026.
\bibitem{rescomp} Rethinking residual errors in compensation-based LLM
quantization. arXiv:2604.07955, 2026.
\bibitem{binaryconnect} M. Courbariaux, Y. Bengio, J.-P. David.
BinaryConnect: Training deep neural networks with binary weights during
propagations. arXiv:1511.00363, 2015.
\bibitem{xnor} M. Rastegari, V. Ordonez, J. Redmon, A. Farhadi. XNOR-Net:
ImageNet classification using binary convolutional neural networks.
arXiv:1603.05279, 2016.
\bibitem{bitnet} H. Wang et al. BitNet: Scaling 1-bit transformers for large
language models. arXiv:2310.11453, 2023.
\bibitem{bitnet158} S. Ma et al. The era of 1-bit LLMs: All large language
models are in 1.58 bits. arXiv:2402.17764, 2024.
\bibitem{onebit} Y. Xu et al. OneBit: Towards extremely low-bit large
language models. arXiv:2402.11295, 2024.
\bibitem{billm} W. Huang et al. BiLLM: Pushing the limit of post-training
quantization for LLMs. arXiv:2402.04291, 2024.
\bibitem{fbillm} L. Ma et al. FBI-LLM: Scaling up fully binarized LLMs from
scratch via autoregressive distillation. arXiv:2407.07093, 2024.
\bibitem{proxquant} Y. Bai, Y.-X. Wang, E. Liberty. ProxQuant: Quantized
neural networks via proximal operators. arXiv:1810.00861, 2018.
\bibitem{md} T. Ajanthan, K. Gupta, P. H. S. Torr, R. Hartley, P. K. Dokania.
Mirror descent view for neural network quantization. arXiv:1910.08237, 2019.
\bibitem{quipsharp} A. Tseng, J. Chee, Q. Sun, V. Kuleshov, C. De Sa.
QuIP\#: Even better LLM quantization with Hadamard incoherence and lattice
codebooks. arXiv:2402.04396, 2024.
\bibitem{spinquant} Z. Liu et al. SpinQuant: LLM quantization with learned
rotations. arXiv:2405.16406, 2024.
\bibitem{faqd} Z. Li et al. Feature affinity assisted knowledge distillation
and quantization of deep neural networks on label-free data.
arXiv:2302.10899, 2023.
\bibitem{zeroq} Y. Cai, Z. Yao, Z. Dong, A. Gholami, M. W. Mahoney,
K. Keutzer. ZeroQ: A novel zero shot quantization framework.
arXiv:2001.00281, 2020.
\bibitem{kd} G. Hinton, O. Vinyals, J. Dean. Distilling the knowledge in a
neural network. arXiv:1503.02531, 2015.
\bibitem{cka} S. Kornblith, M. Norouzi, H. Lee, G. Hinton. Similarity of
neural network representations revisited. ICML, 2019.
\bibitem{qwen} Qwen Team. Qwen2.5 technical report. arXiv:2412.15115, 2024.
\bibitem{wikitext} S. Merity, C. Xiong, J. Bradbury, R. Socher. Pointer
sentinel mixture models. arXiv:1609.07843, 2016.
\bibitem{c4} C. Raffel et al. Exploring the limits of transfer learning with
a unified text-to-text transformer. JMLR, 2020.
\end{thebibliography}
\end{document}